\documentclass[runningheads,a4paper]{llncs}

\usepackage{amsmath}

\usepackage{amsfonts}
\usepackage{amssymb}
\usepackage{graphicx}

\usepackage{url}
\urldef{\mailsa}\path|ebauman@markovprocesses.com, kbauman@stern.nyu.edu|
\urldef{\mailsb}\path|anna.kramer, leonie.kunz, christine.reiss, nicole.sator,|
\urldef{\mailsc}\path|erika.siebert-cole, peter.strasser, lncs}@springer.com|

\begin{document}

\mainmatter  

\title{Discovering the Graph Structure in the Clustering Results}

\titlerunning{Discovering the Graph Structure in the Clustering Results}

%
%
\author{Evgeny Bauman\thanks{Markov Processes Inc.}
\and Konstantin Bauman\thanks{Stern School of Business New York University.}}

\authorrunning{E. Bauman, K. Bauman}

\institute{
\mailsa\\
}

%
%

\toctitle{Discovering the Graph Structure in the Clustering Results}
\maketitle

\begin{abstract}
In a standard cluster analysis, such as k-means, in addition to clusters locations and distances between them, it's important to know if they are connected or well separated from each other. The main focus of this paper is discovering the relations between the resulting clusters. We propose a new method which is based on pairwise overlapping k-means clustering, that in addition to means of clusters provides the graph structure of their relations. The proposed method has a set of parameters that can be tuned in order to control the sensitivity of the model and the desired relative size of the pairwise overlapping interval between means of two adjacent clusters, i.e., {\it level of overlapping.} We present the exact formula for calculating that parameter. The empirical study presented in the paper demonstrates that our approach works well not only on toy data but also compliments standard clustering results with a reasonable graph structure on real datasets, such as financial indices and restaurants.
\end{abstract}

\section{Introduction}\label{intro}

The traditional clustering problem consists of assigning each element to a single cluster such that similar elements are grouped into the same cluster.  One of the most popular clustering algorithms is $k$-means. The main idea of this algorithm is to assign each element to the cluster with the nearest mean, serving as a prototype of the resulting cluster. K-means is a classical algorithm that is widely applied and works well in most of the real data-mining problems.

In additional to standard clustering results, that include the locations of clusters' means and elements assignment, it is useful to find the relations between clusters. Some applications might benefit from the knowledge of this graph of clusters' relations. For example, it could be used in a news categorization problem for the recommendations purposes. If we know which categories are interesting for a particular user, we might recommend a news article from the related category. Another example is biological data, where the graph of clusters relations would help to discover hidden relations between genes.

The simplest way to construct the relations between clusters is to compute the standard euclidean distance measure between the means of clusters in the feature space. However, such distance does not reflect the actual relation hidden in the data. Another way to discover such relations is to find if there is an overlapping between these two clusters or they are well separated.

\begin{figure}
\centering
\includegraphics[width=0.99\linewidth]{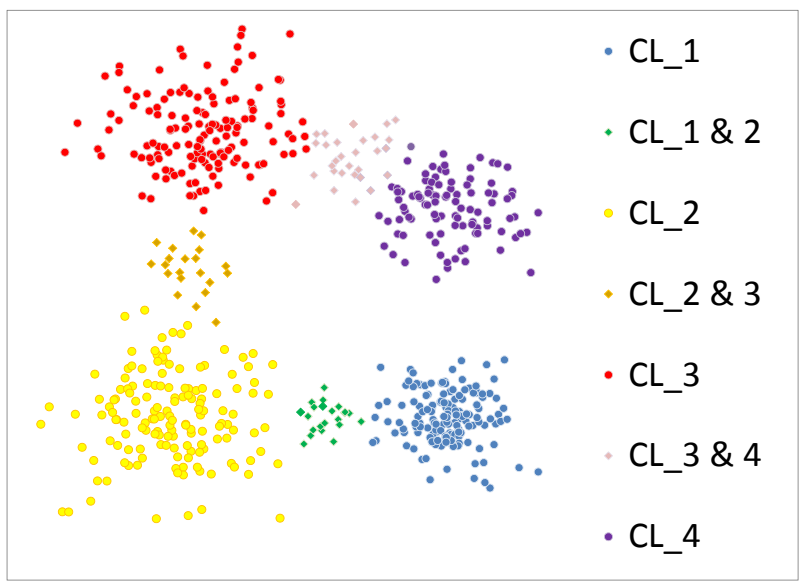}
\caption{Example of Clusters with pairwise Overlapping}\label{synthetic_clusters}
\end{figure}

\begin{figure}
\centering
\includegraphics[width=0.59\linewidth]{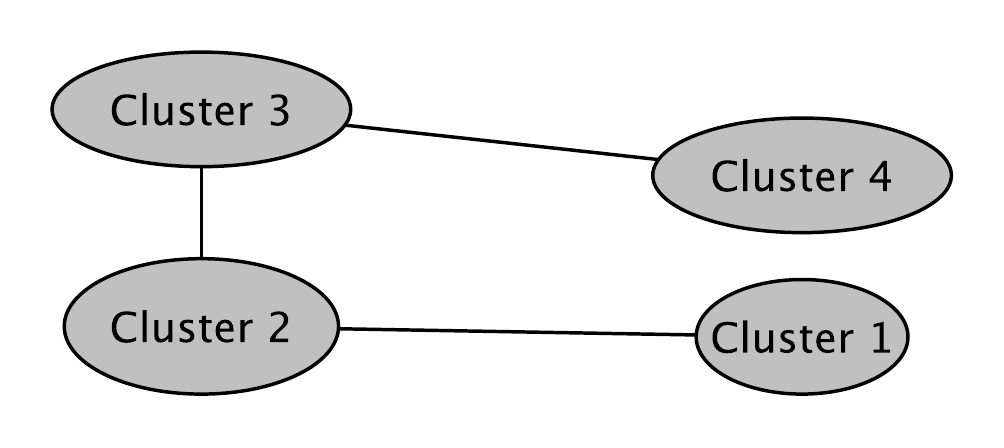}
\caption{Graph of Relations between the Clusters in the Example from Figure \ref{synthetic_clusters}}\label{synthetic_graph}
\end{figure}

Figure~\ref{synthetic_clusters} shows the synthetic example of clustering elements on a plane. In this example, we can see four main clusters. There are some additional points between clusters $C_1$ and $C_2$, $C_2$ and $C_3$ and $C_4$, that can be assigned either to the first cluster or to the second one. In other words, these pairs of clusters have pairwise overlapping between them. The graph of relations which is build based on these overlapping between clusters is presented in Figure~\ref{synthetic_graph}. Note that in this example the distance between clusters $C_1$ and $C_4$ is smaller than the distance between clusters $C_2$ and $C_3$. However, clusters $C_1$ and $C_4$ are well separated and have no elements in their overlapping, thus, they are not connected in the graph (Figure~\ref{synthetic_graph}), while clusters $C_2$ and $C_3$ are connected because they have lots of elements in their overlapping. 

This synthetic example shows that constructing the graph structure of the clusters based on their overlapping can discover hidden information that can be missed by the standard euclidean distance measure.

This problem can be addressed by overlapping clustering. However, in the ordinal overlapping clustering, each object can be assigned to a certain number of clusters. In case if an object assigned to a large number of clusters, it usually means that this object is far from all clusters and does not contribute much to their means. Thus, the overlapping between a large number of clusters makes it difficult to analyze the relations between clusters. Therefore, for the purpose of solving the particular problem of discovering clusters relations, it is reasonable to restrict the maximal number of clusters to which algorithm can assign each object. The easiest way is to set the threshold to two clusters. In this case, pairwise overlapping between clusters can be interpreted as edges in the graph of clusters relations.

In this paper, we focus on the problem of discovering pairwise relations between clusters based on their overlapping. We propose a {\it pairwise overlapping} modification of the $k$-means that allows to assign each element to {\it only} one or two clusters. Therefore, in additional to standard clustering results, our method provides a graph of clusters relations based on the pairwise overlapping between them. The proposed optimization algorithm uses the advantages of  $k$-means approach. In particular, it has an objective function and alternating between ``Assignment'' and ``Update'' steps it converges to the local minimum of its objective function in a finite number of steps. 

In addition, our pairwise overlapping clustering algorithm allows defining the parameter that specifies the level of overlapping between the pairs of resulting clusters. We present a formula for calculating a parameter of the proposed algorithm based on the desired relative size of overlapping interval between means of two adjacent clusters, i.e., {\it level of overlapping}.

In order to show the effectiveness of the proposed algorithm, we tested it on two types of data. In particular, we present the results of applying this algorithm to the problem of constructing Hedge Funds indices and the restaurant's categorization problem. We show that the proposed algorithm produces adequate and easy-interpretable results. In both applications, our method discovered a reasonable graph structure of the resulting clusters. We also provide interpretations of the obtained results in terms of other aspects of objects, such as funds' strategies or restaurants word descriptions.

The rest of the paper is organized as follows. In Section \ref{prior} we discuss the prior work in the domain of overlapping clustering. In Section \ref{algorithm} we present our new pairwise overlapping clustering algorithm. Section \ref{exp} demonstrates the results of applying our algorithm to the financial and restaurants data. And finally, Section~\ref{conc} concludes our findings.

\section{Prior work}\label{prior}
The overlapping clustering problem experienced extensive growth since it was introduced in 1971 by Jardine and Sibson in \cite{2_Jardine}. 
One of the most popular directions of constructing overlapping clustering is formulated as a graph decomposition problem that was studied in such papers as \cite{7_Andersen:2012:OCD:2124295.2124330,6_Khandekar:2012:AOC:2247370.2247412}, where authors solve the problem of minimization graph's conductance, \cite{8_Szalay-Beko:2012:MPC:2348041.2348049} determines overlapping network module hierarchy, \cite{9_Gregory:2008:FAF:1431932.1431981} finds overlapping communities in networks, or \cite{10_Obadi:2010:TRS:1913791.1913897} that presents hierarchical clustering algorithm.
The next important group of overlapping clustering methods is based on the probabilistic approach that was studied in such papers as \cite{11_Meila:2001:ECM:599609.599627} that presents the Naive Bayes Model, \cite{12_Battle:2004:PDO:974614.974637} that proposes the probabilistic relational models (PRMs),  \cite{13_Banerjee:2005:MOC:1081870.1081932,14_Fu:2009:BOS:1674659.1676996} generalizes mixture model method to any other exponential distribution, \cite{15_Shafiei:2006:LDC:1193207.1193349} presents the Multiple Cause Mixture Model.

Furthermore, in \cite{17_Cleuziou,18_Cleuziou:2013:OMB:2422997.2423064} authors proposed a modification of $k$-means for constructing overlapping clustering. This algorithm is based on the idea to use centers of not only single clusters but also groups of a certain number of clusters, such that each element assigned to a group of clusters minimizing the objective function. The method presented in \cite{17_Cleuziou,18_Cleuziou:2013:OMB:2422997.2423064} operates with certain heuristics to find the optimal value of the objective function.

Finally, \cite{whang2015nonexhaus} proposes an objective function that can be viewed as a reformulation of the traditional k-means objective, with easy-to-understand parameters that capture the degrees of overlap and non-exhaustiveness. Authors present iterative algorithm which they call NEO-K-Means (Non-Exhaustive Over- lapping K-Means).

In comparison to all these previous works in overlapping clustering, we propose the pairwise overlapping clustering algorithm that focuses on the particular problem of discovering pairwise relations between clusters based on their pairwise overlapping. In particular, we restrict the maximal number of clusters to which we can assign an element by two and, therefore, allow {\it only} pairwise overlapping between clusters.

There is some prior work on constructing the graph of relations between resulting clusters that was proposed in \cite{5_Bauman}. However, the clusters graph that they construct inherits the relations from the initial graph of the elements. While, the algorithm proposed in this paper constructs the graph of clusters, where the presence or absence of an edge between two clusters shows if they are connected or well separated from each other.

In conclusion, although there is prior work on constructing overlapping clustering, the proposed algorithm is the first that focuses on the particular problem of discovering the relations between clusters based on their pairwise overlapping.

In the next section, we present the specifics of our algorithm.


\section{Discovering the Clusters Graph Structure Using Pairwise Overlapping Method}\label{algorithm}

We consider a problem of {\it pairwise overlapping} clustering on a finite set of elements to $k$ clusters in order to discover the graph structure of the resulting clusters. We define that there is an edge between two clusters $c_i$ and $c_j$ (graph vertexes) if the number of elements in the pairwise overlapping between $c_i$ and $c_j$ exceeds a threshold specified by the number of elements in the minimal cluster from $c_i$ and $c_j$, i.e. $|c_i\cap c_j| > \gamma \cdot  min(|c_i|;|c_j|)$, where $\gamma$ is a parameter that can be set according to the desired sensitivity level of the model. 
In the absence of knowledge about the experimentation domain, parameter $\gamma$ is commonly set to $0.05$ or $0.1$.

In the rest of the section, we describe the algorithm for constructing the pairwise overlapping clustering. 

\subsection{Pairwise Overlapping Clustering: State of the Problem.}

Let  $X = \{x_1,x_2,\dots,x_N\}$ be a finite set of $n$-dimensional vectors $x_j\in R^n, ~ j = 1,\dots,N.$ {\it {Pairwise overlapping}} clustering $\mathcal{H}$ is specified by the assignment matrix $H = ||h_{i,j}||_{k,N}$, where 
$$h_{i,j} = 
\begin{cases} 
1, & \mbox{if } x_j \mbox{ belongs to cluster } c_i \\
0, & \mbox{otherwise,} 
\end{cases}
$$
$$\text{and } 1\leq \sum_{i=1}^k h_{i,j} \leq 2, j=\overline{1,N}.$$

Therefore, each object $x_j$ belongs to one or two clusters from $\mathcal{H}$. 

We assume that each cluster $c_i \in \mathcal{H}$ is described by a certain prototype (mean) $\alpha_i$ -- $n$-dimensional vector, which further will be chosen by optimization of the objective function.
Therefore, the problem of constructing a pairwise overlapping clustering on a set of $N$ elements to $k$ clusters constitutes identifying matrix $H=||h_{i,j}||_{k,N}$ and set of vectors (means) $A=(\alpha_1,\dots,\alpha_k )$ that minimize the following objective function:

\begin{multline}\label{criterion}
J(H;A)=\sum_{j=1}^N\left[\frac{\sum_{i=1}^k (x_j-\alpha_i)^2\cdot h_{i,j}}{\left (\sum_{i=1}^k h_{i,j}\right )^m}\right ], \text{where }\\
h_{i,j}\in \{0;1\}, i=\overline{1,k}; 
\text{ and } 1\leq \sum_{i=1}^k h_{i,j} \leq 2,~j=\overline{1,N}.
\end{multline}

The main idea of the criterion \ref{criterion} is to optimize the sum of the average square distances from each element to the centers of clusters that it belongs to. Note that there are only one or two non-zero summands in the numerator and in the denominator of formula \ref{criterion}.
Parameter $m$ in \ref{criterion} determines the level of overlapping between clusters in the optimal clustering. For example, if  $m=1$  then the optimal clustering should be a partition of the set  $X$. Increasing parameter $m$ leads to increase of the uncertainty in the resulting clustering.

\begin{theorem}\label{optimalTheorem}
For a given finite set $X = \{x_1,\dots,x_N\}$ of $n$-dimensional vectors $x_j\in R^n,$ if matrix $H^*=||h_{i,j}^*||$ and set $A^*=\{\alpha_1^*,\dots,\alpha_k^* \}$ are the optimal matrix and the optimal set of means for the objective function $J(H;A)$ in form of the equation \ref{criterion}, then

\begin{enumerate}
\item for each element $x_j$ and two closest means $\alpha_{i_1}^*\in A^*$  and  $\alpha_{i_2}^*\in A^*$, where $(x_j-\alpha_{i_1}^* )^2<(x_j-\alpha_{i_2}^* )^2$, matrix $H^*$ should satisfy the following conditions:
\begin{itemize}
\item $h_{i_1,j}^*=1,$ $h_{i_2,j}^*=0$ ($x_j$ belongs to cluster $c_{i_1}$), if  $(x_j-\alpha_{i_1}^* )^2 < \frac{(x_j-\alpha_{i_1}^* )^2+(x_j-\alpha_{i_2}^* )^2} {2^m}$

\item $h_{i_1,j}^* = h_{i_2,j}^* = 1$ ($x_j$ belongs to $c_{i_1}$ and $c_{i_2}$), \\if  $(x_j-\alpha_{i_1}^* )^2\geq \frac{(x_j-\alpha_{i_1}^* )^2+(x_j-\alpha_{i_2}^* )^2}{2^m}$

\item $h_{i,j}^*=0$, if  $ i \notin \{i_1, i_2\}.$
\end{itemize}

\item means of the clusters $\alpha_i^*\in A^*$ satisfy the following equation:
\begin{equation}\label{mean}
\alpha_i^*=\frac{\sum_{j=1}^N x_j \frac{h_{i,j}^*}{\left (\sum_{t=1}^k h_{t,j}^* \right )^m}}
{\sum_{j=1}^N \frac{h_{i,j}^*}{\left (\sum_{t=1}^k h_{t,j}^* \right )^m}}
\end{equation}

\end{enumerate}

\end{theorem}

\begin{proof}
 The proof consists of two parts:
\begin{enumerate}
\item Each object $x_j$ has it's corresponding part in the objective function \ref{criterion}. If we have fixed means the problem of assigning $x_j$ to the optimal clusters that minimize objective function \ref{criterion} can be done independently for each summand, which is actually done by the rules specified in the first part of the theorem.

\item The optimal $\alpha_i^*$ should satisfy the equation: $\frac{\partial J(H;A)}{\partial \alpha_i} = 0.$ Therefore, we come up to the following equation: 
$$2\alpha_i^*\sum_{j=1}^N\frac{h_{i,j}^*}{(\sum_{t=1}^k h_{t,j}^*)^m} - 2\sum_{j=1}^N x_j \frac{h_{i,j}^*}{(\sum_{t=1}^k h_{t,j}^* )^m} =0,$$
that gives us formula \ref{mean} for the optimal mean $\alpha_i^*$.

\end{enumerate}
\end{proof}

\subsection{Pairwise Overlapping Clustering Algorithm.}

The presented algorithm has the same structure as the well-known $k$-means algorithm.
It uses an iterative refinement technique. Starting with an initial set of $k$ means $A^{(0)} = (\alpha_1^{(0)},\alpha_2^{(0)},...,\alpha_k^{(0)}),$ the algorithm proceeds by alternating between {\it Assignment} and {\it Update} steps. The initial set of $k$ means can be specified randomly or by some heuristics.

\smallskip

{\bf Assignment step.} Within the $t$-th iteration of the overlapping clustering algorithm on the Assignment step, we fix the values of $k$ means $A^{(t-1)} = \left(\alpha_1^{(t-1)},\alpha_2^{(t-1)},\dots,\alpha_k^{(t-1)}\right)$ from the previous iteration $(t-1)$ and minimize the objective function $J\left(H;A^{(t-1)}\right)$ by finding the optimal matrix $H^{(t)}=||h^{(t)}_{i,j}||$, i.e. by assigning elements to the optimal number of closest clusters.

For each element $x_j$ optimal weights $h_{i,j}^{(t)}$ should satisfy the equations from the first part of theorem \ref{optimalTheorem}. Therefore, for each element  $x_j\in X$ we proceed with the following steps:
\begin{enumerate}
\item identify two closest means $\alpha^{(t-1)}_{i_1} , \alpha^{(t-1)}_{i_2} \in A^{(t-1)}$, where $(x_j-\alpha^{(t-1)}_{i_1})^2<(x_j-\alpha^{(t-1)}_{i_2})^2$
\item set weights $h^{(t)}_{i,j}$ according the following rules:
\begin{itemize}
\item $h_{i_1,j}^{(t)}=1,$ $h_{i_2,j}^{(t)}=0$ (assign $x_j$ to $c_{i_1}^{(t-1)}$), \\if  $(x_j-\alpha_{i_1}^{(t-1)} )^2 < \frac{(x_j-\alpha_{i_1}^{(t-1)} )^2+(x_j-\alpha_{i_2}^{(t-1)} )^2} {2^m}$
\item $h_{i_1,j}^{(t)} = h_{i_2,j}^{(t)} = 1$ (assign $x_j$ to $c_{i_1}^{(t-1)},~c_{i_2}^{(t-1)}$), \\if  $(x_j-\alpha_{i_1}^{(t-1)} )^2\geq \frac{(x_j-\alpha_{i_1}^{(t-1)} )^2+(x_j-\alpha_{i_2}^{(t-1)} )^2}{2^m}$
\item $h_{i,j}^{(t)}=0$, if  $ i \notin \{i_1, i_2\}.$
\end{itemize}

\end{enumerate}

\smallskip

{\bf Update step.} Within the $t$-th iteration of the overlapping clustering algorithm on the Update step we fix the matrix $H^{(t)}$ obtained on the Assignment step and minimize the objective function  $J(H^{(t)};A)$ by finding optimal values of $A^{(t)}$.

According to formula \ref{mean} and similarly to $k$-means clustering algorithm \cite{1_Bezdek:1981:PRF:539444} we set $\alpha_i^{(t)}$ to the mean of the cluster $c_i^{(t)}$ using the following formula: 
\begin{equation}
\alpha_i^{(t)}=\frac{\sum_{j=1}^N x_j \frac{h_{i,j}^{(t)}}{\left (\sum_{t=1}^k h_{t,j}^{(t)} \right )^m}}
{\sum_{j=1}^N \frac{h_{i,j}^{(t)}}{\left (\sum_{t=1}^k h_{t,j}^{(t)} \right )^m}}
\end{equation}

The proposed pairwise overlapping clustering algorithm terminates when the Assignment step and the Update step stop changing the coverage and means of the clusters.

\smallskip
\begin{theorem}\label{convergence}
The proposed pairwise overlapping clustering algorithm converges to a certain local minimum of the objective function \ref{criterion} in a finite number of steps.
\end{theorem}
\begin{proof}
Both the Assignment and the Update steps of the algorithm reduce the objective function \ref{criterion} until it reaches a local minimum. Since the set of the all possible pairwise overlapping clusterings is finite, then the algorithm converges in a finite number of steps.
\end{proof}


\subsection{Setting the Overlapping Level.}\label{overlap_level}

Parameter $m$ in the pairwise overlapping clustering objective function~\ref{criterion} determines the degree of overlapping between the resulting clusters. If $m=1$, then the optimal clustering will be a partition of the set $X$. In the case of $m\rightarrow \infty$, most of the elements $x_j\in X$ will be assigned to a pair of the resulting clusters. Therefore, the question is how to set an appropriate value of $m$ in order to get the desired level of overlapping between clusters.

Let's consider a pair of adjacent clusters $(c_1, c_2)$ in the optimal pairwise overlapping clustering $\mathcal{H}$ and assume that all elements $x_j\in X$ belonging to the interval $I=[\alpha_1,\alpha_2]$ between the means of this clusters $\alpha_1$ and $\alpha_2$, belong either to $c_1$ or to $c_2$ or to the overlap of $c_1$ and $c_2$. We denote by interval $I_1$ points of the interval $I$ belonging to $c_1$, by $I_2$ points of the interval $I$ belonging to $c_2$, and by $I_{1,2}$ points of the interval $I$ belonging to the overlap between $c_1$ and $c_2$. 

For the points $x\in I$ we define the following functions:
$g_1(x) = (x - \alpha_1)^2;$
$g_2(x) = (x - \alpha_2)^2;$
$g_{1,2}(x) = \left(\frac12\right)^m\left((x - \alpha_1)^2+(x - \alpha_2)^2\right).$
According to the assignment step of the overlapping clustering algorithm for point $x$ we claim the following: (a) if $g_1(x)\leq g_{1,2}(x)$ then $x$ belongs to $c_1$; (b) if $g_2(x)\leq g_{1,2}(x)$  then $x$ belongs to $c_2$; (c) if $g_1(x) > g_{1,2}(x)$ and $g_2(x) > g_{1,2}(x)$  then $x$  belongs to the overlap of clusters $c_1$ and $c_2$.
 
Further, we calculate lengths $l(I_1)$, $l(I_2)$ and $l(I_{1,2})$ of specified intervals $I_1$, $I_2$ and $I_{1,2}$ by solving the following equations: $g_1(x) = g_{1,2}(x)$ and $g_2(x) = g_{1,2}(x).$
As a result we get:
$$l(I_1)=l(I_2)=\frac{1}{1+\sqrt{2^m-1}}l(I)$$
$$\mbox{ and }~l(I_{1,2})=\left(1-\frac{2}{1+\sqrt{2^m-1}}\right)l(I).$$

The relative length of overlapping interval $I_{1,2}$ is equal
$$r_{overlap}=\frac{l(I_{1,2})}{l(I)}=\left(1-\frac{2}{1+\sqrt{2^m-1}}\right).$$
Therefore, parameter $m$ can be represented in the following form:
\begin{equation}\label{param_m}
m = \log_2\left(\left(\frac{1+r_{overlap}}{1-r_{overlap}}\right)^2+1\right).
\end{equation}

Formula~\ref{param_m} determines parameter $m$ for the pairwise overlapping clustering algorithm based on the desired relative size of the overlapping interval between the means of two adjacent clusters, i.e., {\it level of overlapping}. For example, 
\begin{itemize}
\item in order to get $r_{overlap} =\frac13$ (in this case $l(I_1)=l(I_2)=l(I_{1,2})$) we should set $m=\log_2 5 \approx 2.33$
 
\item in order to get $r_{overlap} =\frac12$ we should set $m=\log_2 10 \approx 3.32$
 
\item in order to get $r_{overlap} = 0$ (the hard clustering) we should set $m=\log_2 2 = 1$.
\end{itemize}

Usually, in the absence of experimentation or domain knowledge, $m$ is commonly set to $2$ or $3$. In these cases, the level of overlapping would be equal $0.268$ and $0.415$ respectively.

\section{Experiments}\label{exp}

In order to demonstrate how well our algorithm of discovering graph structure of the resulting clusters works in practice, we tested it on two types of applications. The first one is the problem of constructing Hedge Funds Indices and the second one is the restaurant categorization problem. We present the experimental settings and the results for these two applications in sections~\ref{finance_experiment} and \ref{restaurant_experiment} respectively.

\subsection{Discovering the Relations between Hedge Funds Indeces.}\label{finance_experiment}

One of the most important problems of Hedge Funds research is the problem of constructing Hedge Funds Indices. In particular, Hedge Funds Research Inc.\footnote{www.hedgefundresearch.com} works on this problem and constructed a variety of aforementioned indices. 

Most of the indices are constructed in the following way: 1) identify a certain homogeneous market segment, and 2) construct an index as an average value of the key assets from this segment. Therefore, the process of constructing adequate indices that describe the market is reduced to building a good segmentation of the market and computing the centers of these segments. These centers are considered as the indices. Since these macro indices represent means of their clusters, they have more stable and predictable behavior than individual funds.
One of the most important problems in the study of financial indices is the problem of predicting their values. The discovered relations between financial indices might contribute to this prediction problem. The most common way to calculate those relations based on the correlations between indices. However, the connections that are established based on the pairwise overlapping between indices are more stable and are not depend on the temporal state of the financial market. 

\begin{table*}
\centering
\caption{Matrix $M_{strategy}$}\label{m_strategy}
\begin{tabular}{|l|c|c|c|c|c|c|c|c|c|} \hline
Strategy&$All$&$C_1$&$C_2$&$C_3$&$C_4$&$C_5$&$C_6$&$C_7$&$C_8$\\ \hline
Equity Market Neutral&11\%&3\%&1\%&0\%    &8\%&0\%&0\%&4\%&{\bf33\%}\\ \hline
Fundamental Growth&28\%&19\%&{\bf36\%}&{\bf64\%}&22\%&{\bf75\%}&27\%&{\bf37\%}&14\%\\ \hline
Fundamental Value&39\%&{\bf63\%}&25\%&15\%&{\bf48\%}&15\%&{\bf53\%}&35\%&32\%\\ \hline
Energy/Basic Materials&5\%&1\%&{\bf35\%}&15\%&2\%&1\%&2\%&8\%&4\%\\ \hline
Technology/ Healthcare&6\%&8\%&1\%&4\%&{\bf12\%}&1\%&1\%&6\%&7\%\\ 
\hline\end{tabular}
\end{table*}

\begin{table*}
\centering
\caption{Matrix $M_{RIF}$}\label{m_rif}
\begin{tabular}{|l|c|c|c|c|c|c|c|c|c|} \hline
RIF&$All$&$C_1$&$C_2$&$C_3$&$C_4$&$C_5$&$C_6$&$C_7$&$C_8$\\ \hline
North America&48\%&{\bf60\%}&28\%&9\%&{\bf56\%}&4\%&{\bf72\%}&45\%&48\%\\ \hline
Asia ex-Japan&9\%&5\%&6\%&9\%&4\%&{\bf72\%}&4\%&9\%&4\%\\ \hline
Russia/ Eastern Europe&4\%&1\%&{\bf21\%}&{\bf38\%}&1\%&1\%&1\%&2\%&1\%\\ \hline
Western Europe/UK&5\%&2\%&6\%&0\%&7\%&0\%&1\%&4\%&{\bf11\%}\\
\hline\end{tabular}
\end{table*}

In this study, we applied the proposed algorithm of pairwise overlapping clustering to the Hedge Funds data in order to identify the relations between constructed indices. 
As the source of data we use HFR Database\footnote{www.hedgefundresearch.com/index.php?fuse=hfrdb}. We collect the set of all Hedge Funds that use the ``Equity Hedge Strategy", where ``Equity Hedge Strategy" means that they maintain positions both long and short in primary equity and equity derivative securities. Overall in our data, we have $855$ monthly time-series of returns for $855$ funds over the period of time from $06/2007$ till $05/2010$ ($36$ months in total).

First, we run the pairwise overlapping clustering algorithm on the time-series data using returns for each month as individual features. We set the number of clusters to $k=8$, the level of overlapping to $r_{overlap}=\frac13$, and parameter $\gamma = 0.1$. As a result, we get (a) clusters of Hedge Funds, (b) the centers of the clusters that can be interpreted as Hedge Funds macro indices, and (c) the graph of relations between constructed indices. 

In order to show that our algorithm provides an adequate separation of funds into clusters we compare the results or the pairwise overlapping clustering with two attributes of Hedge Funds: (a) strategies that are actually sub-strategies of ``Equity Hedge Strategy"; (b) Regional Investment Focus. 
First, based on the $5$ main strategy types (out of $8$ in total) and for the $8$ resulting clusters we build $5\times8$ matrix $M_{strategy}$ of correspondence between strategy types and clusters, that is presented in Table \ref{m_strategy}. Each entry $(i, C_j)$ in this matrix contains the number of funds that have the $i$-th strategy type and correspond to the cluster $C_j$, normalized by the total number of funds in cluster $C_j$. The first column $All$ of matrix $M_{strategy}$ contains the numbers of funds having $i$-th strategy type normalized by the total number of funds. Entry $(i,C_j)$ is marked in bold if it's significantly higher than $(i,All)$. In this case, funds from cluster $C_j$ use the $i$-th strategy type more often than funds corresponding to other clusters. For example, we can say that $75\%$ of funds corresponding to cluster $C_5$ use ``Fundamental Growth" strategy.

Further, $4\times8$ matrix $M_{RIF}$ presented in Table~\ref{m_rif} shows the correspondence between the Regional Investment Focus (RIF) and clusters. We calculate it for the main $4$ RIF types (out of $13$ in total) and on the $8$ resulting clusters. 

Based on the matrices $M_{strategy}$ and $M_{RIF}$ we can define an interpretation of clusters in terms of strategy types and Regional Investment Focuses.
For instance, funds that belong to clusters $C_1$, $C_4$ and $C_6$ mainly use ``Fundamental Value" strategy and their Region Investment Focus is mainly in North America. However, funds from cluster $C_4$ in addition use ``Technology/Healthcare" strategy. Funds from clusters $C_2$ and $C_3$ use ``Fundamental Growth" strategy and their primary RFI is ``Russia/Eastern Europe". Funds that belong to cluster $C_2$ also use ``Energy/Basic Materials" strategy. Further, funds from clusters $C_5$ and $C_7$ use ``Fundamental Growth" strategy, and funds from $C_5$ have RIF is ``Asia ex-Japan". Finally, funds from cluster $C_8$ mainly use ``Equity Market Neutral" strategy and their RIF is ``Western Europe/UK".
Note that we do not use fund's strategy and RIF features while building the pairwise overlapping clustering. However, our algorithm constructs clusters that appear to be an adequate representation of the funds' separation in terms of their main strategy and their RIFs. 

\begin{figure}
\centering
\includegraphics[width=0.99\linewidth]{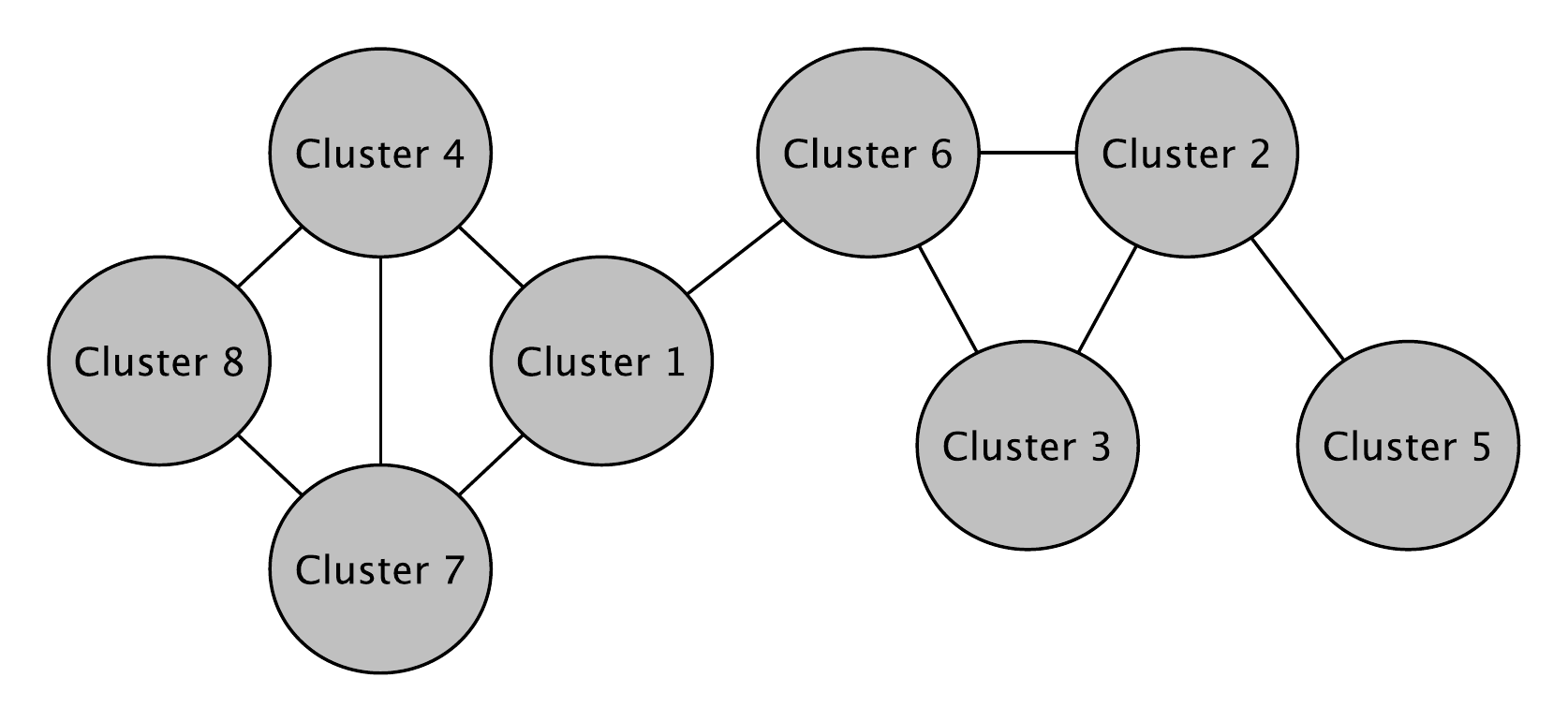}
\caption{Graph of the Pairwise Overlapping between Clusters of Hedge Funds}\label{financial_graph}
\end{figure}

The graph of clusters relations that was discovered based on the pairwise overlapping between clusters is presented in Figure \ref{financial_graph}. 
As we can see, the cluster $C_8$ with focus on ``Western Europe/UK" stands on the left and has connections to the clusters $C_4$ and $C_7$ with focus to ``North America," whereas clusters $C_2$ and $C_3$ with focus on ``Russia / Eastern Europe" stand on the right and have connections to both ``North America" focused cluster $C_6$ and ``Asia ex-Japan" focused cluster $C_5$. 
Further, although clusters $C_4$ and $C_6$ share the same regional investment focus, they are not connected with an edge. It means that these clusters are well separated by our method and have only small number of funds in the overlapping. This difference can be explained in terms of funds' strategy, where in contrast to cluster $C_6$ cluster $C_4$ actively uses ``Technology/Healthcare" strategy. However, not all edges and absences of edges can be explained in the two discussed parameters (Strategy and RIF), which shows that our method discovers additional hidden connections between clusters by analyzing their pairwise overlapping.  

In conclusion, based on the time-series of Hedge Funds returns we build a pairwise overlapping clustering of funds. We showed that the resulting clustering is adequate separation of funds in terms of strategies and regional investment focuses. Moreover, we construct a graph structure of the resulting clusters and showed that the edges in that graph adequately describe the connections between clusters discovering hidden information about their relations.


\subsection{Discovering the Relations between Clusters of Restaurants.}\label{restaurant_experiment}

For the second experiment, we used restaurant application. We address the problem of discovering the relations between restaurant categories that can be built automatically by clustering.
The knowledge of such relations can be useful for the recommendation purposes. For example, ``Italian restaurants" category may have a relation to the ``Fast-food-Pizza" category since there are some restaurants that correspond to both of this categories. In this case, for the user who likes to visit Italian restaurants in the evening, we may recommend Pizzeria at the lunch time. 
 

In our study we used the Yelp\footnote{www.yelp.com} data that was provided for the Yelp Dataset Challenge\footnote{www.yelp.com/dataset\_challenge}. In particular, we used all the reviews that were collected in the Phoenix metropolitan area in Arizona over the period of $6$ years for all the $4503$ restaurants ($158430$ reviews). In addition, all restaurants have a set of specified categories, such as ``Burgers", ``Chinese", ``Sushi Bars" etc. For our study, we selected $36$ different categories that contain at least $50$ restaurants. Further, we applied our algorithm of discovering graph structure to restaurants data as follows. 

Firstly, for each restaurant $r_i$ we collect a set of reviews $S_i$ and clean these reviews from stop-words that are too generic and unlikely help us to identify the restaurant's categories.
We next applied the well-known LDA approach \cite{Blei:2003:LDA:944919.944937} using sets $S_i$ as documents and obtained $40$ topics, representing distributions of words. Some of them directly refer to the restaurant's cuisine, e.g. \{mexican, salsa, taco, beans, tacos\}, \{pita, hummus, greek, feta\}, \{seafood, shrimp, fish, crab\}, but some of them refer to other aspects of user experience in a restaurant, e.g. \{atmosphere, cool, patio, friends, outside, outdoor\}, \{sports, tv, game, football, wings, watch\}. At the end of this step for each restaurant $r_i$ we assign a $40$-dimensional vector according to the distribution of the resulting topics in the set of reviews~$S_i$.

On the next step, we run the proposed pairwise overlapping clustering algorithm on the set of vectors from the previous step using parameter $r_{overlap}=\frac13$. Since our algorithm can converge to a local minimum of criteria function, we ran it $100$ times starting from randomly selected points. Our final result defined as the best result of the objective function~\ref{criterion} from $100$ runs. Finally, we construct a graph of clusters based on their pairwise overlapping using parameter $\gamma = 0.1.$

In order to examine the quality of the pairwise overlapping clustering algorithm in this particular application, we compare the resulting clusters with the categorization of restaurants provided by Yelp. This analysis shows that there are $26$ (out of $36$ in total) categories that have intersections with corresponding clusters in more than $50\%$ of the restaurants, and $7$ of those categories have intersections with corresponding clusters in more than $80\%$ of restaurants. It means that separation constructed by our method is adequate in terms of real categories.

Furthermore, for each cluster of restaurants, we identify the set of the most important features based on the distribution of corresponding topics discussed above using some threshold level. Therefore, each cluster is described in a set of $1 - 5$ topics.

\begin{figure}
\centering
\includegraphics[width=0.99\linewidth]{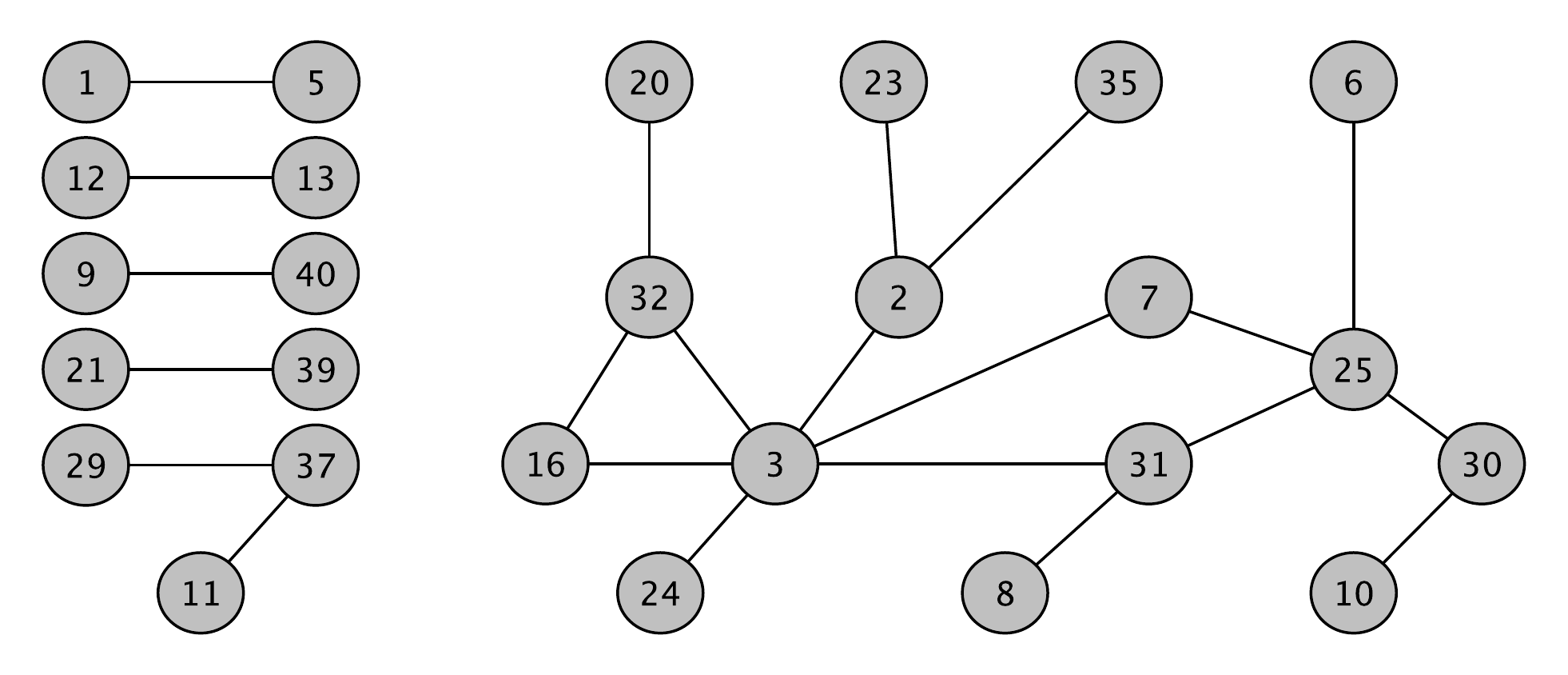}
\caption{Graph of the Pairwise Overlapping between Clusters of Restaurants}\label{restaurants_graph}
\end{figure}

The discovered graph of clusters connections is presented in the Figure~\ref{restaurants_graph}. For the simplification, we eliminated the clusters that are too small (have less that $20$ objects) and the clusters that have no connections to other clusters. The presented graph has $6$ components, where most of them represent connections between 2 or 3 clusters. For example, cluster 21, which is described with topics \{coffee, iced, yougurt\} and \{flavors, creamy, fruit\}, connected to the cluster 39, which is described with topics \{flavors, creamy, fruit\} and \{chocolate, vanilla, cake\}. As you can see, this clusters are pretty similar in terms of topics and, therefore they are connected.

The largest component of the discovered graph contains 15 clusters. Our method identified that they are not strongly connected, but there are some connections through other clusters. For example,  clusters 3, 31 and 25 represent a chain. Clusters 3 and 31 share the same topic \{menu, delicious, restaurant\}, clusters 31 and 25 share topics \{wine, bottle, glass\} and \{server, ordered, table\}, while clusters 3 and 25 have no important topics in common. Although the standard euclidean distance between clusters 3 and 25 ($0.00384$) is less than the distance between clusters 31 and 25 ($0.00496$), these clusters are not connected with an edge. It shows that our algorithm of discovering relations between clusters differs the simplest approach based on the euclidean distance. 

In conclusion, we constructed a clustering of restaurants based on the words that are used in the corresponding reviews and also discovered a graph of relations between the resulting clusters. The separation is adequate in terms of standard categorization and the clusters graph adequately represents clear connections and discovers some hidden ones.

\section{Conclusion}\label{conc}

In this paper, we presented a new algorithm for discovering graph of relations between clusters. In particular, we proposed a {\it pairwise overlapping} clustering algorithm that focuses on this particular discovery problem. This algorithm is a modification of the $k$-means that allows to assign each element to {\it only} one or two clusters. 
We constructed the corresponding optimization algorithm and proved that alternating between ``Assignment'' and ``Update'' steps it converges to a certain local minimum of the objective function in a finite number of steps.

Furthermore, the presented pairwise overlapping clustering algorithm allows defining the parameter that specifies the level of overlapping between the resulting clusters. We present the formula for calculating this parameter based on the desired relative size of overlapping interval between the means of two adjacent clusters, i.e., {\it level of overlapping}. 

Finally, we tested the effectiveness of the proposed algorithm on two types of data. In particular, we presented the results of applying this algorithm to the problem of constructing Hedge Funds Indices and to the restaurant's categorization problem. We showed that our algorithm produced adequate and easy-interpretable results and discovers a reasonable graph structure of the resulting clusters.


\bibliographystyle{abbrv}
\bibliography{overlap_lib}

\end{document}